%% file: main.tex
\documentclass[sigconf]{acmart}
\usepackage{url}

\usepackage{multirow}
\usepackage{balance}
\usepackage{amsfonts}
\usepackage{amsmath}
\usepackage{amssymb}
\usepackage{amsthm}
\usepackage{graphicx}
\usepackage{microtype}
\usepackage{url}
\usepackage{enumitem}
\usepackage{wrapfig,lipsum}
\usepackage{multirow}
\usepackage{makecell}
\usepackage{wrapfig}
\usepackage{tabularx}
\usepackage{subfigure}
\usepackage[noend,ruled]{algorithm2e}
\usepackage{bbm}
\usepackage{algpseudocode}
\usepackage{etoolbox}
\usepackage{subfigure}
\usepackage{float}
\usepackage{booktabs}
\usepackage{caption}

\newcommand{\bs}[1]{\boldsymbol{#1}}

\DeclareMathOperator*{\argmin}{arg\,min}
\renewcommand{\textrightarrow}{$\rightarrow$}

\newcommand{\wrong}[1]{#1 ($\times$)}
\newcommand{\ie}{\emph{i.e.}}
\newcommand{\eg}{\emph{e.g.}}
\newcommand{\CatEm}{\mbox{\sf CatE}\xspace}

\newcommand{\nop}[1]{}

\theoremstyle{definition}
\newtheorem{definition}{Definition}
\newtheorem{theorem}{Theorem}

\newtheorem{example}{Example}

\AtBeginDocument{%
  \providecommand\BibTeX{{%
    \normalfont B\kern-0.5em{\scshape i\kern-0.25em b}\kern-0.8em\TeX}}}

\setcopyright{acmcopyright}
\begin{document}
\fancyhead{}
\leftmargini=12pt 
\title{Discriminative Topic Mining via Category-Name Guided Text Embedding}

\author{Yu Meng$^{1*}$, Jiaxin Huang$^{1*}$, Guangyuan Wang$^{1}$, Zihan Wang$^{1}$, \\ Chao Zhang$^{2}$, Yu Zhang$^{1}$, Jiawei Han$^{1}$}
\affiliation{
\institution{$^1$Department of Computer Science, University of Illinois at Urbana-Champaign, IL, USA} 
\institution{$^2$College of Computing, Georgia Institute of Technology, GA, USA}
\institution{$^{1}$\{yumeng5, jiaxinh3, gwang10, zihanw2, yuz9, hanj\}@illinois.edu \ \ \ $^2$chaozhang@gatech.edu}
}
\thanks{$^*$Equal Contribution.}
\renewcommand{\shortauthors}{Meng et al.}

\input{0-abs}

\ccsdesc[500]{Information systems~Data mining}
\ccsdesc[300]{Information systems~Document topic models}
\ccsdesc[300]{Information systems~Clustering and classification}

\copyrightyear{2020}
\acmYear{2020}
\setcopyright{iw3c2w3}
\acmConference[WWW '20]{Proceedings of The Web Conference 2020}{April 20--24, 2020}{Taipei, Taiwan}
\acmBooktitle{Proceedings of The Web Conference 2020 (WWW '20), April 20--24, 2020, Taipei, Taiwan}
\acmPrice{}
\acmDOI{10.1145/3366423.3380278}
\acmISBN{978-1-4503-7023-3/20/04}

\keywords{Topic Mining, Discriminative Analysis, Text Embedding, Text Classification}

\maketitle

\input{1-intro}

\input{2-def}

\input{3-emb}

\input{4-cap}

\input{5-sum}

\input{6-exp}

\input{7-related}

\input{8-concl}




\begin{acks}
Research was sponsored in part by DARPA under Agreements No. W911NF-17-C-0099 and FA8750-19-2-1004, National Science Foundation IIS 16-18481, IIS 17-04532, and IIS-17-41317, and DTRA HDTRA11810026. 
Any opinions, findings, and conclusions or recommendations expressed in this document are those of the author(s) and should not be interpreted as the views of any U.S. Government. The U.S. Government is authorized to reproduce and distribute reprints for Government purposes notwithstanding any copyright notation hereon.
We thank anonymous reviewers for valuable and insightful feedback.
\end{acks}

\bibliographystyle{ACM-Reference-Format}
\balance
\bibliography{ref}

\end{document}

%% file: 0-abs.tex

\begin{abstract}
Mining a set of meaningful and distinctive topics automatically from massive text corpora has broad applications.
Existing topic models, however, typically work in a purely unsupervised way, which often generate topics that do not fit users' particular needs and yield suboptimal performance on downstream tasks. 
We propose a new task, discriminative topic mining, which leverages a set of user-provided category names to mine discriminative topics from text corpora. This new task not only helps a user understand clearly and distinctively the topics he/she is most interested in, but also benefits directly keyword-driven classification tasks.
We develop \CatEm, a novel category-name guided text embedding method for discriminative topic mining, which effectively leverages minimal user guidance to learn a discriminative embedding space and discover category representative terms in an iterative manner. We conduct a comprehensive set of experiments to show that \CatEm mines high-quality set of topics guided by category names only, and benefits a variety of downstream applications including weakly-supervised classification and lexical entailment direction identification\footnote{Source code can be found at \url{https://github.com/yumeng5/CatE}.}. 
\end{abstract}

%% file: 1-intro.tex

\section{Introduction}

To help users effectively and efficiently comprehend a large set of text documents, it is of great interest to generate a set of meaningful and coherent topics automatically from a given corpus. 
Topic models~\cite{Blei2003LatentDA,Hofmann1999ProbabilisticLS} are such unsupervised statistical tools that discover latent topics from text corpora. Due to their effectiveness in uncovering hidden semantic structure in text collections, topic models are widely used in text mining~\cite{Foster2007MixtureModelAF,Mei2007AutomaticLO} and information retrieval tasks~\cite{Dou2007ALE,Wei2006LDAbasedDM}.

Despite of their effectiveness, traditional topic models suffer from two noteworthy limitations: (1) \emph{Failure to incorporate user guidance}. Topic models tend to retrieve the most general and prominent topics from a text collection, which may not be of a user's particular interest, or provide a skewed and biased summarization of the corpus. (2) \emph{Failure to enforce distinctiveness among retrieved topics}. Concepts are most effectively interpreted via their uniquely defining features.
For example, Egypt is known for pyramids and China is known for the Great Wall. Topic models, however, do not impose disriminative constraints, resulting in vague interpretations of the retrieved topics.  
Table~\ref{tab:lda_topic}
demonstrates three retrieved topics from the New York Times (\textbf{NYT}) annotated corpus~\cite{Sandhaus2008} via LDA~\cite{Blei2003LatentDA}. 
We can see that it is difficult to clearly define the meaning of the three topics due to an overlap of their semantics (\eg, the term ``united states'' appears in all three topics).

\setlength{\tabcolsep}{3pt}
\begin{table}[h]
	\centering
	\caption{LDA retrieved topics on \textbf{NYT} dataset. The meanings of the retrieved topics have overlap with each other.}
	\vspace*{-1em}
	\label{tab:lda_topic}
	\scalebox{0.95}{
		\begin{tabular}{c|c|c}
			\toprule
			Topic 1 & Topic 2 & Topic 3  \\
			\midrule
			canada, united states & sports, united states & united states, iraq \\
			canadian, economy & olympic, games & government, president \\
			\bottomrule
		\end{tabular}
	}
	\vspace*{-1em}
\end{table}
\setlength{\tabcolsep}{5pt}

In order to incorporate user knowledge or preference into topic discovery for mining distinctive topics from a text corpus, we propose a new task, \textbf{Discriminative Topic Mining}, which takes only a set of category names as user guidance, and aims to retrieve a set of representative and discriminative terms under each provided category. 
In many cases, a user may have a specific set of interested topics in mind, or have prior knowledge about the potential topics in a corpus. 
Such user interest or prior knowledge may come naturally in the form of a set of category names that could be used to guide the topic discovery process, resulting in more desirable results that better cater to a user's need and fit specific downstream applications. 
For example, a user may provide several country names and rely on discriminative topic mining to retrieve each country's provinces, cities, currency, etc. from a text corpus. 
We will show that this new task not only helps the user to clearly and distinctively understand his/her interested topics, but also benefits keywords-driven classification tasks.

There exist previous studies that attempt to incorporate prior knowledge into topic models. Along one line of work, supervised topic models such as Supervised LDA~\cite{mcauliffe2008supervised} and DiscLDA~\cite{lacoste2009disclda} guide the model to predict category labels based on document-level training data. While they do improve the discriminative power of unsupervised topic models on classification tasks, they rely on massive hand-labeled documents, which may be difficult to obtain in practical applications. Along another line of work that is more similar to our setting, users are asked to provide a set of seed words to guide the topic discovery process, which is referred to as seed-guided topic modeling~\cite{Andrzejewski2009LatentDA,Jagarlamudi2012IncorporatingLP}. However, they still do not impose requirements on the distinctiveness of the retrieved topics and thus are not optimized for discriminative topic presentation and other applications such as keyword-driven classification. 

We develop a novel category-name guided text embedding method, \CatEm, for discriminative topic mining. 
\CatEm consists of two modules: (1) A \emph{category-name guided text embedding learning module} that takes a set of category names to learn category distinctive word embeddings by modeling the text generative process conditioned on the user provided categories, and (2) a \emph{category representative word retrieval module} that selects category representative words based on both word embedding similarity and word distributional specificity. The two modules collaborate in an iterative way: 
At each iteration, the former 
refines word embeddings and category embeddings for accurate representative word retrieval; and the latter 
selects representative words that will be used by the former at the next iteration.

Our contributions can be summarized as follows.
\begin{enumerate}
\parskip -0.2ex
\item We propose discriminative topic mining, a new task for topic discovery from text corpora with a set of category names as the only supervision. We show qualitatively and quantitatively that this new task helps users obtain a clear and distinctive understanding of interested topics, and directly benefits keyword-driven classification tasks.

\item We develop a category-name guided text embedding framework for discriminative topic mining by modeling the text generation process. The model effectively learns a category distinctive embedding space that best separates the given set of categories based on word-level supervision.

\item We propose an unsupervised method that jointly learns word embedding and word distributional specificity, which allow us to consider both relatedness and specificity when retrieving category representative terms. We also provide theoretical interpretations of the model.

\item We conduct a comprehensive set of experiments on a variety of tasks including topic mining, weakly-supervised classification and lexical entailment direction identification to demonstrate the effectiveness of our model on these tasks.

\end{enumerate}

%% file: 2-def.tex

\section{Problem Formulation}
\begin{definition} [Discriminative Topic Mining]
Given a text corpus $\mathcal{D}$ and a set of category names $\mathcal{C}=\{c_1, \dots, c_n\}$, discriminative topic mining aims to retrieve a set of terms $\mathcal{S}_i=\{w_1, \dots, w_m\}$ from $\mathcal{D}$ for each category $c_i$ such that each term in $\mathcal{S}_i$ semantically belongs to and only belongs to category $c_i$.
\end{definition}
\begin{example}
Given a set of country names, $c_1$: ``The United States'', $c_2$: ``France'' and $c_3$: ``Canada'', it is correct to retrieve ``Ontario'' as an element in $\mathcal{S}_3$, because Ontario is a province in Canada and exclusively belongs to Canada semantically. However, it is incorrect to retrieve ``North America'' as an element in $\mathcal{S}_3$, because North America is a continent and does not belong to any countries. It is also incorrect to retrieve ``English'' as an element in $\mathcal{S}_3$, because English is also the national language of the United States.
\end{example}

The differences between discriminative topic mining and standard topic modeling are mainly two-fold: (1) Discriminative topic mining requires a set of user provided category names and only focuses on retrieving terms belonging to the given categories. (2) Discriminative topic mining imposes strong discriminative requirements that each retrieved term under the corresponding category must belong to and only belong to that category semantically.

%% file: 3-emb.tex

\section{Category-Name Guided Embedding}

In this section, we first formulate a text generative process under user guidance, and then cast the learning of the generative process as a category-name guided text embedding model. Words, documents and categories are jointly embedded into a shared space where embeddings are not only learned according to the corpus generative assumption, but also encouraged to incorporate category distinctive information.

\subsection{Motivation}

Traditional topic models like LDA~\cite{Blei2003LatentDA} use document-topic and topic-word distributions to model the text generation process, where an obvious defect exists due to the bag-of-words generation assumption---each word is drawn independently from the topic-word distribution without considering the correlations between adjacent words. In addition, topic models make explicit probabilistic assumptions regarding the text generation mechanism, resulting in high model complexity and inflexibility~\cite{Gallagher2017AnchoredCE}. 

Along another line of text representation research, word embeddings like Word2Vec~\cite{Mikolov2013DistributedRO} effectively capture word semantic correlations by mapping words with similar \emph{local contexts} closer in the embedding space. They do not impose particular assumptions on the type of data distribution of the corpus and enjoy greater flexibility and higher efficiency. However, word embeddings usually do not exploit document-level co-occurrences of words (\ie, \emph{global contexts}) and also cannot naturally incorporate latent topics into the model without making topic-relevant generative assumptions.

To take advantage of both lines of work for mining topics from text corpora, we propose to explicitly model the text generation process and cast it as an embedding learning problem.

\subsection{Modeling Text Generation Under User Guidance}

When the user provides $n$ category names, we assume text generation is a three-step process: (1) First, a document $d$ is generated conditioned on one of the $n$ categories (this is similar to the assumption in multi-class classification problems where each document belongs to exactly one of the categories); (2) second, each word $w_i$ is generated conditioned on the semantics of the document $d$; and (3) third, surrounding words $w_{i+j}$ in the local context window ($-h \le j \le h, j \neq 0$, $h$ is the local context window size) of $w_i$ are generated conditioned on the semantics of the center word $w_i$. Step (1) explicitly models the associations between each document and user-interested categories (\ie, \emph{topic assignment}). Step (2) makes sure each word is generated in consistency with the semantics of its belonging document (\ie, \emph{global contexts}). Step (3) models the correlations of adjacent words in the corpus (\ie, \emph{local contexts}). Putting the above pieces together, we have the following expression for the likelihood of corpus generation conditioned on a specific set of user-interested categories $\mathcal{C}$:
\begin{align}
\label{eq:gen}
P(\mathcal{D} \mid \mathcal{C}) &= \prod_{d\in \mathcal{D}} p(d \mid c_d) \prod_{w_i \in d} p(w_i \mid d) \prod_{\substack{w_{i+j} \in d \\ -h \le j \le h, j \neq 0}} p(w_{i+j} \mid w_i)
\vspace*{-1em}
\end{align}
where $c_d$ is the latent category of $d$. 

Taking the negative log-likelihood as our objective $\mathcal{L}$, we have
\begin{align}
\label{eq:obj}
\begin{split}
\mathcal{L} &= -\sum_{d\in \mathcal{D}} \log p(d \mid c_d) \qquad (\mathcal{L}_{\text{topic}})\\
& \quad - \sum_{d\in \mathcal{D}} \sum_{w_i \in d} \log p(w_i \mid d) \qquad (\mathcal{L}_{\text{global}})\\
& \quad - \sum_{d\in \mathcal{D}} \sum_{w_i \in d} \sum_{\substack{w_{i+j} \in d \\ -h \le j \le h, j \neq 0}} \log p(w_{i+j} \mid w_i). \qquad (\mathcal{L}_{\text{local}})\\
\end{split}
\end{align}

In Eq.~\eqref{eq:obj}, $p(w_i \mid d)$ and $p(w_{i+j} \mid w_i)$ are observable (\eg, $p(w_i \mid d) = 1$ if $w_i$ appears in $d$, and $p(w_i \mid d) = 0$ otherwise), while $p(d \mid c_d)$ is latent (\ie, we do not know which category $d$ belongs to). To directly leverage the word level user supervisions (\ie, category names), a natural solution is to decompose $p(d \mid c_d)$ into word-topic distributions:
\begin{align*}
p(d \mid c_d) &\propto p(c_d \mid d) p(d) \propto p(c_d \mid d) \propto \prod_{w \in d} p(c_d \mid w),
\end{align*}
where the first proportionality is derived via Bayes rule; the second derived assuming $p(d)$ is constant; and the third assumes $p(c_d \mid d)$ is jointly decided by all words in $d$.

Next, we rewrite the first term in Eq.~\eqref{eq:obj} (\ie, $\mathcal{L}_{\text{topic}}$) by reorganizing the summation over categories instead of documents:
$$
\mathcal{L}_{\text{topic}} = -\sum_{d\in \mathcal{D}} \log p(d \mid c_d) = -\sum_{c \in \mathcal{C}} \sum_{w \in c} p(c \mid w) + \text{const.}
$$

Now $\mathcal{L}_{\text{topic}}$ is expressed in $p(c \mid w)$, the category assignment of words. This is exactly the task we aim for---finding words that belong to the  categories.

\subsection{Embedding Learning}

In this subsection, we introduce how to formulate the optimization of the objective in Eq.~\eqref{eq:obj} as an embedding learning problem.

Similar to previous work~\cite{Mikolov2013DistributedRO,Bojanowski2017EnrichingWV}, we define the three probability expressions in Eq.~\eqref{eq:obj} via log-linear models in the embedding space:
\begin{align}
\label{eq:distr}
p(c_i \mid w) &= \frac{\exp(\bs{c}_{i}^\top \bs{u}_{w})}{\sum_{c_{j}\in \mathcal{C}} \exp(\bs{c}_{j}^\top \bs{u}_{w})},
\end{align}
\begin{equation}
\label{eq:sg_d}
p(w_{i} \mid d) = \frac{\exp(\bs{u}_{w_{i}} ^\top \bs{d})}{\sum_{d' \in \mathcal{D}}\exp(\bs{u}_{w_i}^\top \bs{d'})},
\end{equation}
\begin{equation}
\label{eq:sg_w}
p(w_{i+j} \mid w_i) = \frac{\exp(\bs{u}_{w_i}^\top \bs{v}_{w_{i+j}})}{\sum_{w' \in V}\exp(\bs{u}_{w_i}^\top \bs{v}_{w'})},
\end{equation}
where $\bs{u}_{w}$ is the input vector representation of $w$ (usually used as the word embedding); $\bs{v}_{w}$ is the output vector representation that serves as $w$'s contextual representation; $\bs{d}$ is the document embedding; $\bs{c}_{i}$ is the category embedding. Please note that Eqs.~\eqref{eq:sg_d} and \eqref{eq:sg_w} are not yet the final design of our embedding model, as we will propose an extension of them in Section~\ref{sec:joint} that leads to a more effective and suitable model for discriminative topic mining.

While Eqs.~\eqref{eq:sg_d} and \eqref{eq:sg_w} can be directly plugged into Eq.~\eqref{eq:obj} to train word and document embeddings, Eq.~\eqref{eq:distr} requires knowledge about the latent topic (\ie, the category that $w$ belongs to) of a word $w$.
Initially, we only know the user-provided category names belong to their corresponding categories, but during the iterative topic mining process, we will retrieve more terms under each category, gradually discovering the latent topic of more words.

To this end, we design the following for learning $\mathcal{L}_{\text{topic}}$ in Eq.~\eqref{eq:obj}: Let $\bs{p}_{w}=\begin{bmatrix} p(c_1 \mid w) & \dots & p(c_n \mid w) \end{bmatrix}^\top$ be the probability distribution of $w$ over all classes. If a word $w$ is known to belong to class $c_i$, $\bs{p}_{w}$ computed from Eq.~\eqref{eq:distr} should become a one-hot vector $\bs{l}_{w}$ (\ie, the category label of $w$) with $p(c_i \mid w)=1$. To achieve this property, we minimize the KL divergence from each category representative word's distribution $\bs{p}_{w}$ to its corresponding discrete delta distribution $\bs{l}_{w}$. Formally, given a set of class representative words $\mathcal{S}_i$ (we will introduce how to retrieve $\mathcal{S}_i$ in Section \ref{sec:concen}) for category $c_i$, the $\mathcal{L}_{\text{topic}}$ term is implemented as:

\begin{align}
\label{eq:reg}
\mathcal{L}_{\text{topic}} &= \sum_{c_i \in \mathcal{C}} \sum_{w \in \mathcal{S}_i} KL\left(\bs{l}_{w} \middle\| \bs{p}_{w} \right).
\end{align}

From the embedding learning perspective, Eq.~\eqref{eq:reg} is equivalent to a cross-entropy regularization loss, encouraging the category embeddings to become distinctive anchor points in the embedding space that are far from each other and are surrounded by their current retrieved class representative terms.

%% file: 4-cap.tex

\section{Category Representative Word Retrieval}
\label{sec:concen}
In this section, we detail how to retrieve category representative words (\ie, the words that belong to and only belong to a category) for topic mining.

As a starting point, we propose to retrieve category representative terms by jointly considering two separate aspects: Relatedness and specificity. In particular, a representative word $w$ of category $c$ should satisfy simultaneously two constraints: (1) $w$ is semantically related to $c$, and (2) $w$ is semantically more specific than the category name of $c$. Constraint (1) can be imposed by simply requiring high cosine similarity between a candidate word embedding and the category embedding. However, constraint (2) is not naturally captured by the text embedding space. Hence, we are motivated to improve the previous text embedding model by incorporating word specificity signals. 

In the following, we first present the concept of word distributional specificity, and then introduce how to capture the signal effectively in our model. Finally, we describe how to retrieve category representative words by jointly considering the two constraints.

\subsection{Word Distributional Specificity}
We adapt the concept of distributional generality in \cite{Weeds2004CharacterisingMO} and define word distributional specificity as below.
\begin{definition}[Word Distributional Specificity]
We assume there is a scalar $\kappa_w \ge 0$ correlated with each word $w$ indicating how specific the word meaning is. The bigger $\kappa_w$ is, the more specific meaning word $w$ has, and the less varying contexts $w$ appears in.
\end{definition}
The above definition is grounded on the distributional inclusion hypothesis~\cite{ZhitomirskyGeffet2005TheDI} which states that hyponyms are expected to occur in a subset of the contexts of their hypernyms.

For example, ``\textit{seafood}'' has a higher word distributional specificity than ``\textit{food}'', because seafood is a specific type of food.

\subsection{Jointly Learning Word Embedding and Distributional Specificity}
\label{sec:joint}
In this subsection, we propose an extension of Eqs.~\eqref{eq:sg_d} and \eqref{eq:sg_w} to jointly learn word embedding and word distributional specificity in an \emph{unsupervised} way.

Specifically, we modify Eqs.~\eqref{eq:sg_d} and \eqref{eq:sg_w} to incorporate an additional learnable scalar $\kappa_{w}$ for each word $w$, while constraining the embeddings to be on the unit hyper-sphere $\mathbb{S}^{p-1} \subset \mathbb{R}^{p}$, motivated by the fact that directional similarity is more effective in capturing semantics~\cite{meng2019spherical}.

Formally, we re-define the probability expressions in Eqs.~\eqref{eq:sg_d} and \eqref{eq:sg_w} to be\footnote{Eq.~\eqref{eq:distr} is not refined with the $\kappa$ parameter because we do not aim to learn category specificity.}:
\begin{equation}
\label{eq:pg}
p(w_i \mid d) = \frac{\exp(\kappa_{w_i}\bs{u}_{w_i}^\top \bs{d})}{\sum_{d' \in \mathcal{D}}\exp(\kappa_{w_i}\bs{u}_{w_i}^\top \bs{d'})},
\end{equation}
\begin{equation}
\label{eq:pl}
p(w_{i+j} \mid w_i) = \frac{\exp(\kappa_{w_i}\bs{u}_{w_i}^\top \bs{v}_{w_{i+j}})}{\sum_{w' \in V}\exp(\kappa_{w_i}\bs{u}_{w_i}^\top \bs{v}_{w'})},
\end{equation}
$$
s.t. \quad \forall w, d, c, \quad \|\bs{u}_{w}\|=\|\bs{v}_{w}\|=\|\bs{d}\|=\|\bs{c}\|=1.
$$
In practice, the unit norm constraints can be satisfied by simply normalizing the embedding vectors after each update\footnote{Alternatively, one may apply the Riemannian optimization techniques in the spherical space as described in ~\cite{meng2019spherical}.}. Under the above setting, the parameter $\kappa_{w}$ learned is the distributional specificity of $w$.

\subsection{Explaining the Model}
\label{sec:explain}

We explain here why the additional parameter $\kappa_{w}$ in Eqs.~\eqref{eq:pg} and \eqref{eq:pl} effectively captures word distributional specificity.
We first introduce a spherical distribution, and then show how our model is connected to the properties of the distribution.

\begin{definition} [The von Mises Fisher (vMF) distribution]
A unit random vector $\bs{x} \in \mathbb{S}^{p-1} \subset \mathbb{R}^{p}$
has the $p$-variate von Mises Fisher distribution $vMF_p(\bs{\mu},\kappa)$ if its probability dense function is

\begin{equation*}
\label{eq:vmf}
f(\bs{x};\bs{\mu},\kappa) = c_p(\kappa)\exp (\kappa\bs{\mu}^\top\bs{x}),
\end{equation*}
where $\kappa \ge 0$ is the concentration parameter, $\|\bs{\mu}\| = 1$ is the mean direction, and the normalization constant $c_p(\kappa)$ is given by

$$
c_p(\kappa) = \frac{\kappa^{p/2-1}}{(2\pi)^{p/2} I_{p/2-1}(\kappa)},
$$
where $I_r(\cdot)$ represents the modified Bessel function of the first kind at order $r$.
\end{definition}

\begin{theorem}
\label{thm:main}
When the corpus size and vocabulary size are infinite (\ie, $|\mathcal{D}| \to \infty$ and $|V| \to \infty$) and all $p$-dimensional word vectors and document vectors are unit vectors, generalizing Eqs.~\eqref{eq:pg} and \eqref{eq:pl} to the continuous cases results in the $p$-variate vMF distribution with the center word vector $\bs{u}_{w_i}$ as the mean direction and $\kappa_{w_i}$ as the concentration parameter, \ie,

\begin{equation}
\label{eq:vmf_w}
\lim_{|V| \to \infty} p(w_{i+j} \mid w_i) = c_p(\kappa_{w_i})\exp (\kappa_{w_i} \bs{u}_{w_i}^\top \bs{v}_{w_{i+j}}),
\end{equation}

\begin{equation}
\label{eq:vmf_d}
\lim_{|V| \to \infty} p(w_i \mid d) = c_p(\kappa_{w_i})\exp (\kappa_{w_i} \bs{u}_{w_i}^\top \bs{d}).
\end{equation}

\end{theorem}

\begin{proof}
We give the proof for Eq.~\eqref{eq:vmf_w}. The proof for Eq.~\eqref{eq:vmf_d} can be derived similarly.

We generalize the relationship proportionality $p(w_{i+j} \mid w_i) \propto \exp (\kappa_{w_i}\bs{u}_{w_i}^\top\bs{v}_{w_{i+j}})$ in Eq.\ (\ref{eq:pl}) to the continuous case and obtain the following probability density distribution:

\begin{align}
\label{eq:claim}
\begin{split}
\lim_{|V| \to \infty} p(w_{i+j} \mid w_i) &= \frac{\exp(\kappa_{w_i}\bs{u}_{w_i}^\top \bs{v}_{w_{i+j}})}{\int_{\mathbb{S}^{p-1}}\exp(\kappa_{w_i}\bs{u}_{w_i}^\top \bs{v}_{w'}) d \bs{v}_{w'}} \\
& \triangleq \frac{\exp(\kappa_{w_i}\bs{u}_{w_i}^\top \bs{v}_{w_{i+j}})}{Z},
\end{split}
\end{align}
where $Z$ denotes the integral in the denominator.

The probability density function of vMF distribution integrates to $1$ over the entire sphere, \ie, 

$$
\int_{\mathbb{S}^{p-1}} c_p(\kappa_{w_i}) \exp(\kappa_{w_i}\bs{u}_{w_i}^\top \bs{v}_{w'}) d \bs{v}_{w'} = 1,
$$

we have 
$$
Z = \int_{\mathbb{S}^{p-1}}\exp(\kappa_{w_i}\bs{u}_{w_i}^\top \bs{v}_{w'}) d \bs{v}_{w'} = \frac{1}{c_p(\kappa_{w_i})}.
$$

Plugging $Z$ back to Eq.\ (\ref{eq:claim}), we obtain
\begin{equation*}
\lim_{|V| \to \infty} p(w_{i+j} \mid w_i) = c_p(\kappa_{w_i})\exp (\kappa_{w_i} \bs{u}_{w_i}^\top \bs{v}_{w_{i+j}}).
\end{equation*}

\vspace*{-1em}
\end{proof}

Theorem \ref{thm:main} reveals the underlying generative assumption of the joint learning model defined in Section~\ref{sec:joint}---the contexts vectors are assumed to be generated from the vMF distribution with the center word vector $\bs{u}_{w_i}$ as the mean direction and $\kappa_{w_i}$ as the concentration parameter. Our model essentially learns both word embedding and word distributional specificity that maximize the probability of the context vectors getting generated by the center word's vMF distribution.
\begin{figure}[ht]
\centering
\includegraphics[width=0.95\linewidth]{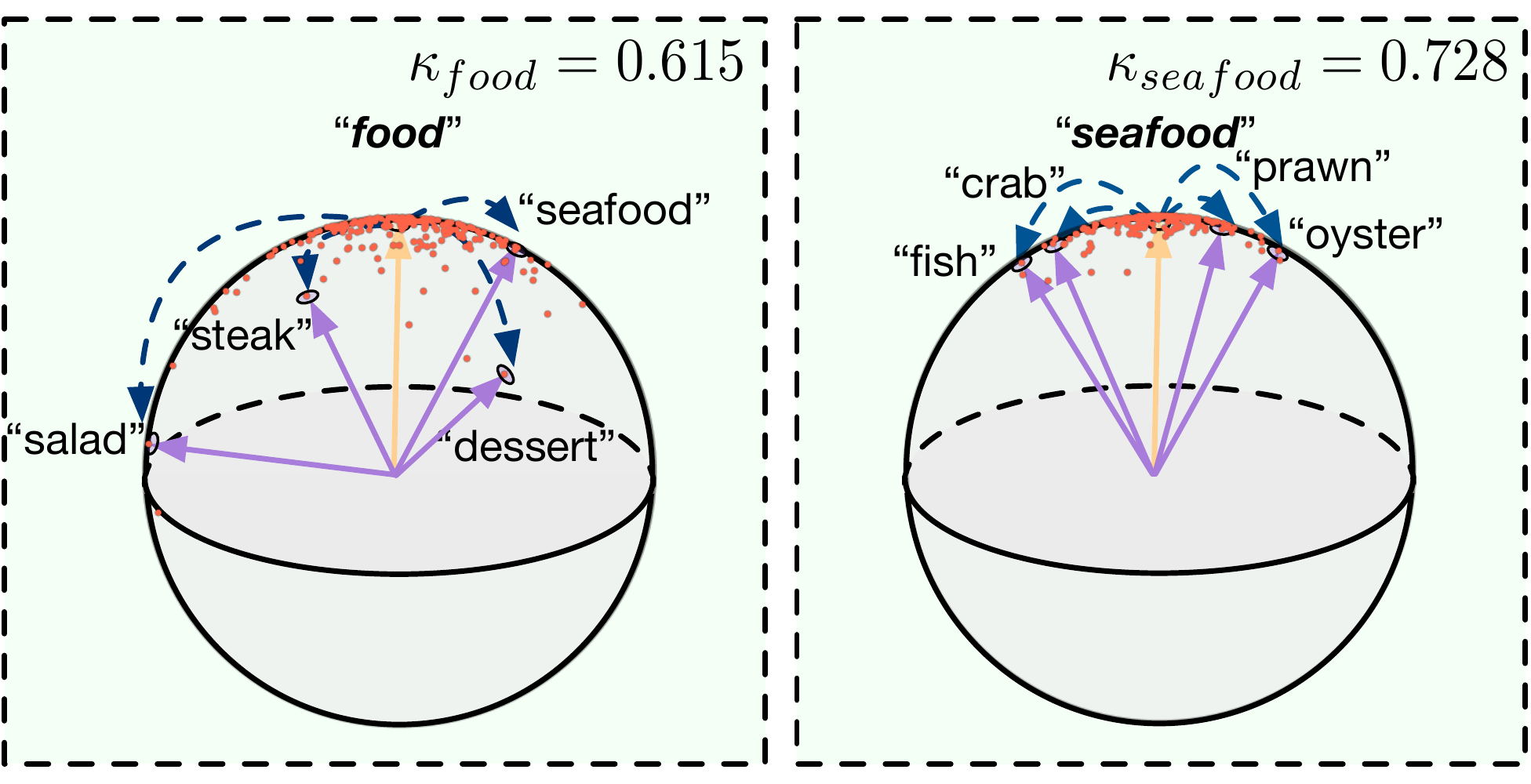}
\vspace*{-1em}
\caption{Word Distributional Specificity.}
\label{fig:vmf}
\vspace*{-1em}
\end{figure}
Figure \ref{fig:vmf} shows two words with different distributional specificity. ``\textit{Food}'' has more general meaning than ``\textit{seafood}'' and appears in more diverse contexts. Therefore, the learned vMF distribution of ``\textit{food}'' will have a lower concentration parameter than that of ``\textit{seafood}''. In other words, ``\textit{food}'' has a lower distributional specificity than ``\textit{seafood}''.

\subsection{Selecting Category Representative Words}
Finally, the learned distributional specificity can be used to impose the constraint that class representative words should belong to the category. Specifically, a category representative word must have higher distributional specificity than the category name.
However, we also want to avoid selecting too specific terms as category representative words.
From the embedding learning perspective, words with higher semantic specificity may appear fewer times in the corpus and suffer from lower embedding quality and higher variance due to insufficient training, which can lead to the distortion of the category embedding manifold if they are selected as category representative words.

Therefore, among all the words that are more specific than the category name, we prefer words that (1) have high embedding cosine similarity with the category name, and (2) have low distributional specificity, which indicates wider semantic coverage. Formally, we find a representative word of category $c_i$ and add it to the set $\mathcal{S}$ by

\begin{equation}
\label{eq:select}
\begin{gathered}
w = \argmin_w \text{rank}_{sim}(w, c_i) \cdot \text{rank}_{spec}(w)\\
s.t. \quad w \notin \mathcal{S} \quad \text{and} \quad \kappa_w > \kappa_{c_i},
\end{gathered}
\end{equation}
where $\text{rank}_{sim}(w, c_i)$ is the ranking of $w$ by embedding cosine similarity with category $c_i$, \ie, $\cos(\bs{u}_{w}, \bs{c}_{i})$, from high to low; $\text{rank}_{spec}(w)$ is the ranking of $w$ by distributional specificity, \ie, $\kappa_w$, from low to high.

%% file: 5-sum.tex

\subsection{Overall Algorithm}

We summarize the overall algorithm of discriminative topic mining in Algorithm~\ref{alg:train}.
\begin{algorithm}[h]
\caption{Discriminative Topic Mining.}
\label{alg:train}
\KwIn{
A text corpus $\mathcal{D}$; a set of category names $\mathcal{C} = \{c_{i}\}|_{i=1}^{n}$.
}
\KwOut{Discriminative topic mining results $\mathcal{S}_i|_{i=1}^{n}$.}

\For{$i \gets 1$ to $n$} {
$\mathcal{S}_i \gets \{c_{i}\}$\Comment{initialize $\mathcal{S}_i$ with category names}\;
}
\For{$t \gets 1$ to $max\_iter$}  {
Train $\bs{\mathcal{W}}, \bs{\mathcal{C}}$ on $\mathcal{D}$ according to Equation (\ref{eq:obj})\;
\For{$i \gets 1$ to $n$} {
$w \gets$ Select representative word of $c_i$ by Eq.~\eqref{eq:select}\;
$\mathcal{S}_i \gets \mathcal{S}_i \cup \{w\}$\;
}
}
\For{$i \gets 1$ to $n$} {
$\mathcal{S}_i \gets \mathcal{S}_i \setminus \{c_{i}\}$\Comment{exclude category names}\;
}
Return $\mathcal{S}_i|_{i=1}^{n}$\;
\end{algorithm}

Initially, the set of class representative words $\mathcal{S}_i$ is simply the category name.
During training, $\mathcal{S}_i$ gradually incorporates more class representative words so that the category embedding models more accurate and complete class semantics. The embeddings of class representative words are directly enforced by Eq.\ (\ref{eq:reg}) to encode category distinctive information, and this weak supervision signal will pass to other words through Eqs.\ (\ref{eq:pg}) and (\ref{eq:pl}) so that the resulting embedding space is specifically fine-tuned to distinguish the given set of categories.

%% file: 6-exp.tex

\section{Experiments}



\subsection{Experiment Setup}\label{setup}

\textbf{Datasets.}
We use two datasets, the New York Times annotated corpus (\textbf{NYT}) \cite{Sandhaus2008}, the recently released \textit{Yelp Dataset Challenge} (\textbf{Yelp})\footnote{https://www.yelp.com/dataset/challenge}. \textbf{NYT} and \textbf{Yelp} each has two sets of categories: \textbf{NYT}:  \emph{topic} and \emph{location}; \textbf{Yelp}: \emph{food type} and \emph{sentiment}.
For \textbf{NYT}, we first select the major categories (with more than $100$ documents) from topics and locations, and then collect documents that are single-labeled on both set of categories, \ie, each document has exactly one ground truth topic label and one ground truth location label. We do the same for \textbf{Yelp}. The category names and the number of documents in each category can be found in Figure \ref{fig:dataset}.

\begin{figure}[t]
\vspace{-0.3cm}
\subfigure{
\label{fig:nyt_topic}
\includegraphics[width = 0.22\textwidth]{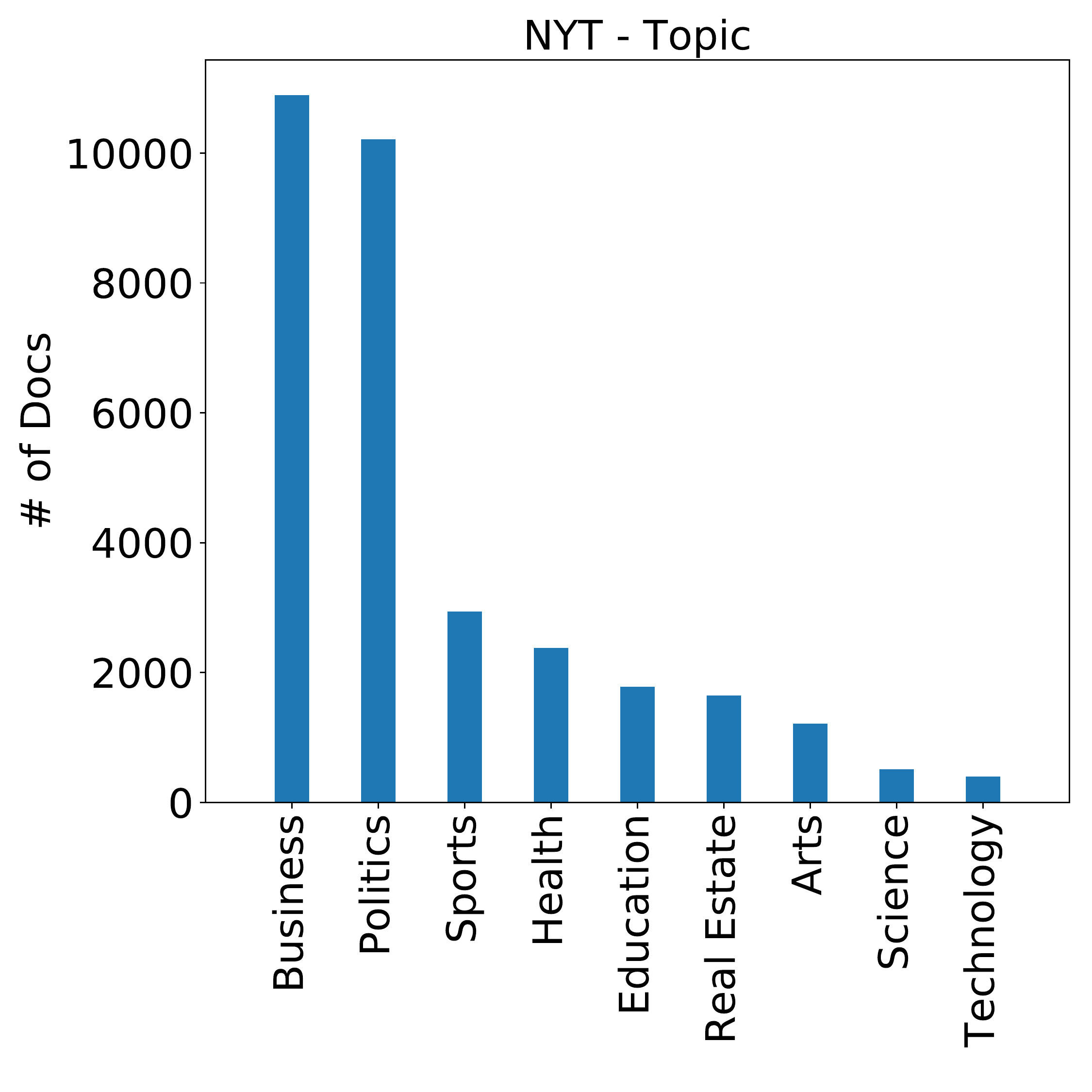}
}
\subfigure{
\label{fig:nyt_loc}
\includegraphics[width = 0.22\textwidth]{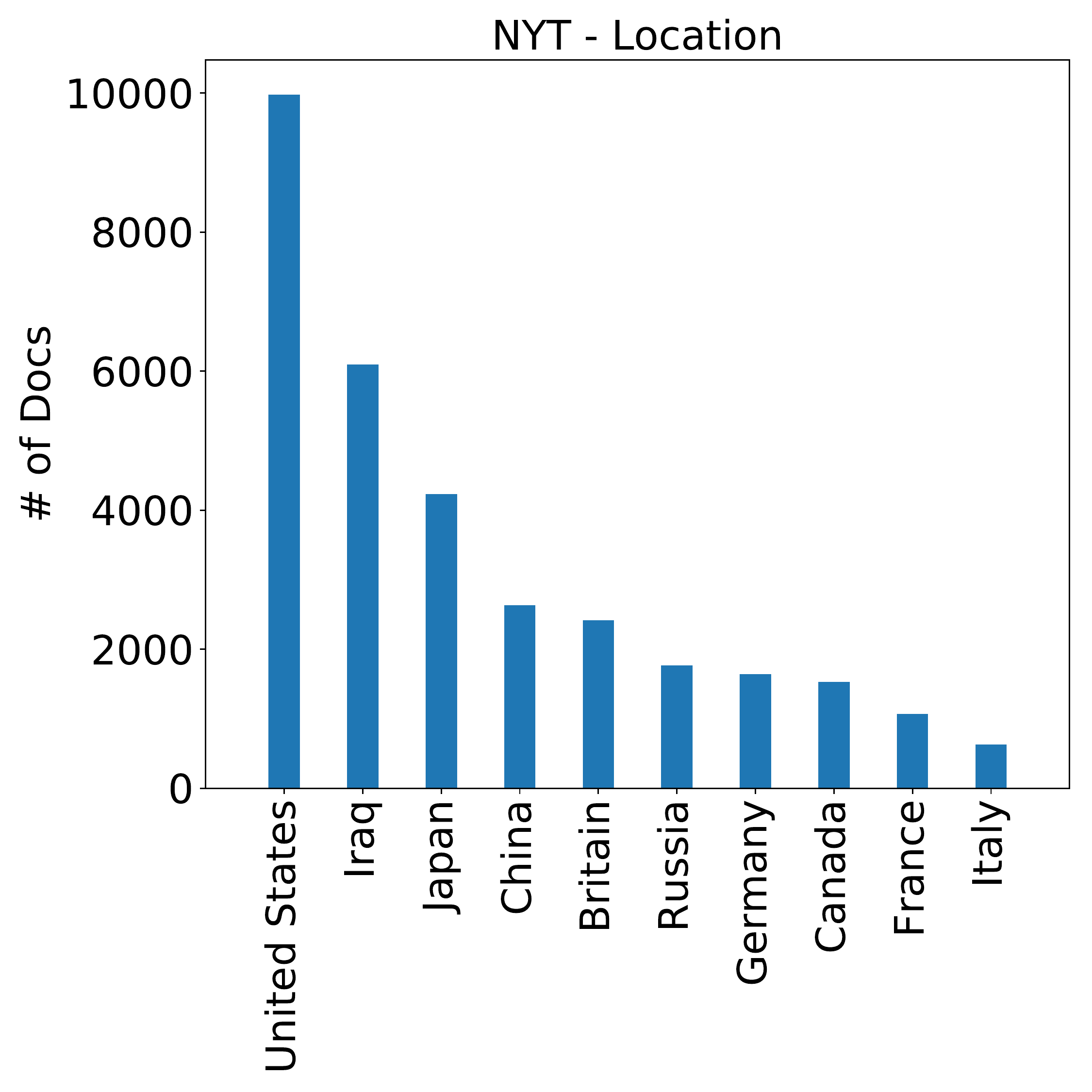}
}

\subfigure{
\label{fig:yelp_food}
\includegraphics[width = 0.22\textwidth]{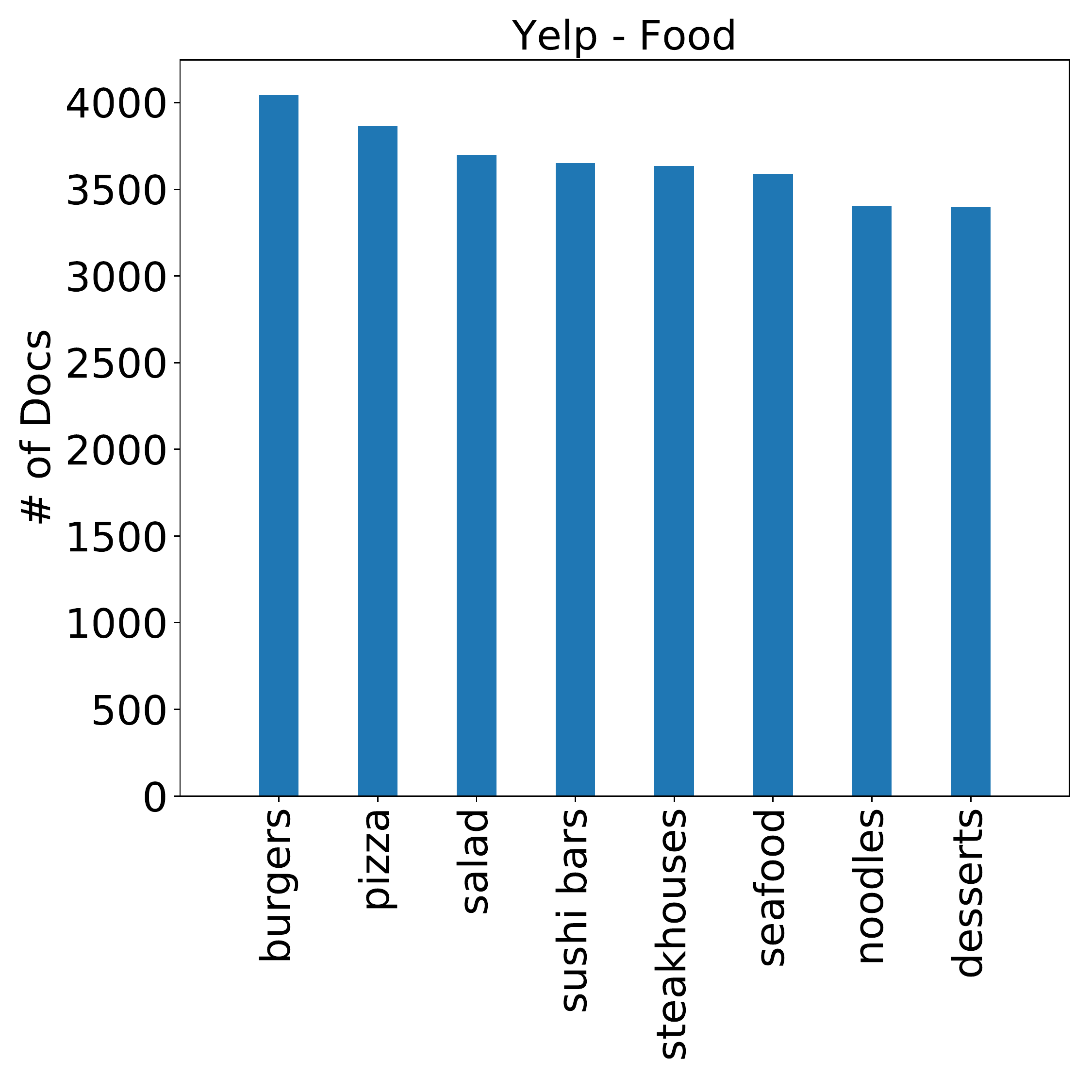}
}
\subfigure{
\label{fig:yelp_senti}
\includegraphics[width = 0.22\textwidth]{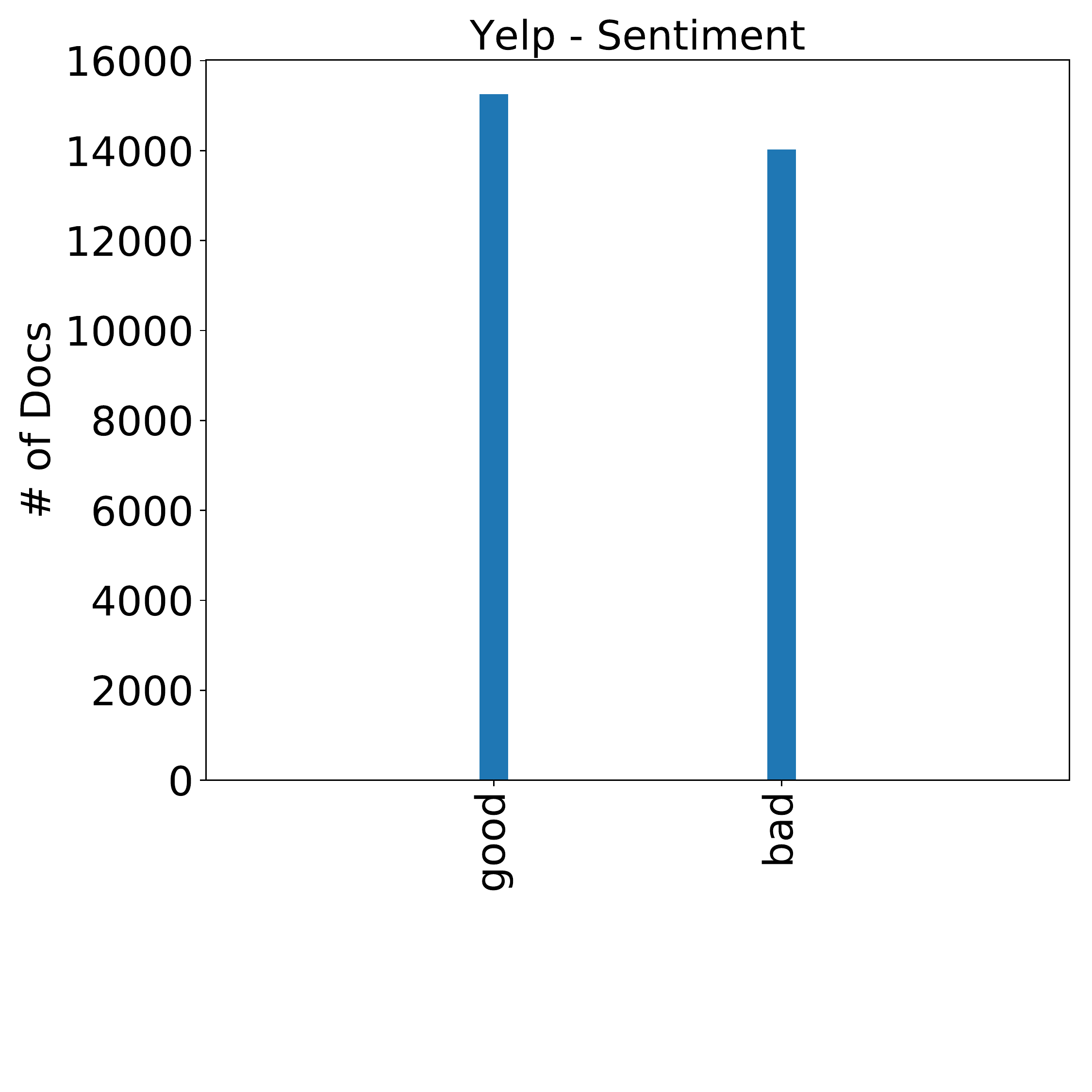}
}
\vspace*{-1em}
\caption{Dataset statistics.}
\label{fig:dataset}
\vspace*{-1.5em}
\end{figure}

\noindent
\textbf{Implementation Details and Parameters.}
Since the full softmax in Eqs.~\eqref{eq:pg} and \eqref{eq:pl} results in computational complexity proportional to the vocabulary size, we adopt the negative sampling strategy \cite{Mikolov2013DistributedRO} for efficient approximation. The training objective (Eq.~\eqref{eq:obj}) is optimized with SGD.
We pre-process the corpus by discarding infrequent words that appear less than $5$ times in the corpus. We use AutoPhrase~\cite{Shang2018AutomatedPM} to extract quality phrases, and the phrases are treated as single words during embedding training.
For fair comparisons with baselines, we set the hyperparameters as below for all methods: word embedding dimension $p=100$, local context window size $h=5$, number of negative samples $k=5$, training iterations on the corpus $max\_iter=10$. Other parameters (if any) are set to be the default values of the corresponding algorithm. In \CatEm, the word distributional specificity parameter $\kappa_w$ is initialized to $1$ for each word.

\subsection{Discriminative Topic Mining}

\noindent
\textbf{Compared Methods.} We compare \CatEm with the following baselines including traditional topic modeling, seed-guided topic modeling and embedding-based topic modeling. For all the baseline methods that require the number of topics $N$ as input, we vary $N$ in $[n, 2n, \dots, 10n]$ where $n$ is the actual number of categories, and report the best performance of the method. Note that our method \CatEm does not require any tuning of $N$ and directly uses the provided category names as the only supervision.
\begin{itemize}[leftmargin=*] 
\item LDA~\cite{Blei2003LatentDA}: LDA is the standard topic model that learns topic-word and document-topic distributions by modeling the generative process of the corpus. It is unsupervised and cannot incorporate seed words as supervision. We manually select the most relevant topics to the provided category names.
\item Seeded LDA~\cite{Jagarlamudi2012IncorporatingLP}: Seeded LDA biases the regular topic model generative process by introducing a seed topic distribution induced by input seed words to encourage the model to focus on user-interested topics. We provide the category names as seed words.
\item TWE~\cite{Liu2015TopicalWE}: TWE
has three models for learning word embedding under a set of topics. For all three models, we use the topic specific word representation of the category names to retrieve representative phrases under each category, and report the best performance of the three models.
\item Anchored CorEx~\cite{Gallagher2017AnchoredCE}:
CorEx
does not rely on generative assumptions and learns maximally informative topics measured by total correlation. Anchored CorEx incorporates user-provided seed words by balancing between compressing the original corpus and preserving anchor words related information. We provide the category names as seed words.
\item Labeled ETM~\cite{Dieng2019TopicMI}: ETM
uses the distributed representation of word embedding to enhance the robustness of topic models to rare words. We use the labeled ETM version which is more robust to stop words. The top phrases are retrieved according to embedding similarity with the category name.
\item \CatEm: Our proposed method retrieves category representative terms according to both embedding similarity and distributional specificity, as described by Eq.~\eqref{eq:select}.

\end{itemize}

\noindent
\textbf{Evaluation Metrics.} We apply two metrics on the top-$m$ ($m=10$ in our experiments) words/phrases retrieved under each category to evaluate all methods:
\begin{itemize}[leftmargin=*] 
\item Topic coherence (TC) is a standard metric~\cite{Lau2014MachineRT} in topic modeling which measures the coherence of terms inside each topic, and is defined as the average normalized pointwise mutual information of two words randomly drawn from the same document, \ie,
$$
\text{TC} = \frac{1}{\binom{m}{2}\cdot n} \sum_{k=1}^n \sum_{i=1}^{m} \sum_{j=i+1}^{m} -\frac{\log \frac{P(w_i, w_j)}{P(w_i) P(w_j)}}{\log P(w_i, w_j)},
$$
where $P(w_i, w_j)$ is the probability of $w_i$ and $w_j$ co-occurring in a document; $P(w_i)$ is the marginal probability of $w_i$.
\item Mean accuracy (MACC) measures the proportion of retrieved top words that actually belong to the category defined by user-provided category names, \ie,
$$
\text{MACC} = \frac{1}{n} \sum_{k=1}^n \frac{1}{m} \sum_{i=1}^{m} \mathbbm{1}(w_i \in c_k),
$$
where $\mathbbm{1}(w_i \in c_k)$ is the indicator function of whether $w_i$ belongs to category $c_k$. 
We invite five graduate students to independently label whether each retrieved term belongs to the corresponding category. The final results are the averaged labeling of the five annotators. 
\end{itemize}

\noindent 
\textbf{Results.} We show both qualitative and quantitative evaluation results. We randomly select two categories from \textbf{NYT}-Location, \textbf{NYT}-Topic, \textbf{Yelp}-Food Type and \textbf{Yelp}-Sentiment respectively, and show top-5 words/phrases retrieved by all methods under each category in Table~\ref{tab:quality}. Terms that are determined by more than half of the human annotators to not belong to the corresponding category are marked with $(\times)$. We measure the topic modeling quality by TC and MACC across all categories and report the results in Table~\ref{tab:quantity}.

\setlength{\tabcolsep}{3pt}
\begin{table*}[ht]
	\small
	\centering
	\caption{Qualitative evaluation on discriminative topic mining.}
	\vspace*{-1em}
	\label{tab:quality}
	\scalebox{0.95}{
		\begin{tabular}{c|cc|cc|cc|cc}
			\toprule
			\multirow{2}{*}{Methods} &
			\multicolumn{2}{c|}{\textbf{NYT}-Location} &
			\multicolumn{2}{c|}{\textbf{NYT}-Topic} & \multicolumn{2}{c|}{\textbf{Yelp}-Food} &
			\multicolumn{2}{c}{\textbf{Yelp}-Sentiment} \\
			& britain & canada & education & politics & burger & desserts & good & bad \\
			\midrule
			
			\multirow{5}{*}{\makecell{LDA}} & \wrong{company} & \wrong{percent} & school & campaign & fatburger & ice cream & great & \wrong{valet} \\
			& \wrong{companies} & \wrong{economy} & students & clinton & \wrong{dos} & chocolate & \wrong{place} & \wrong{peter} \\
			& british & canadian & \wrong{city} & mayor & \wrong{liar} & gelato & love & \wrong{aid} \\
			& \wrong{shares} & \wrong{united states} & \wrong{state} & election & cheeseburgers & \wrong{tea} & friendly & \wrong{relief} \\
			& great britain & \wrong{trade} & schools & political & \wrong{bearing} & sweet & breakfast & rowdy\\
			\midrule
			
			\multirow{5}{*}{\makecell{Seeded\\LDA}} & british & \wrong{city} & \wrong{state} & republican & \wrong{like} & \wrong{great} & \wrong{place} & \wrong{service} \\
			& \wrong{industry} & \wrong{building} & school & political & fries & \wrong{like} & great & \wrong{did} \\
			& \wrong{deal} & \wrong{street} & students & senator & \wrong{just} & ice cream & \wrong{service} & \wrong{order} \\
			& \wrong{billion} & \wrong{buildings} & \wrong{city} & president & \wrong{great} & \wrong{delicious} & 
			\wrong{just} & \wrong{time} \\
			& \wrong{business} & \wrong{york} & \wrong{board} & democrats & \wrong{time} & \wrong{just} & \wrong{ordered} & \wrong{ordered} \\
			\midrule
			
			\multirow{5}{*}{\makecell{TWE}} & \wrong{germany} & toronto & \wrong{arts} & religion & burgers & chocolate & tasty & subpar \\
			& \wrong{spain} & \wrong{osaka} & fourth graders & race & fries & \wrong{complimentary} & decent & \wrong{positive} \\
			& \wrong{manufacturing} & \wrong{booming} & \wrong{musicians} & \wrong{attraction} & hamburger & \wrong{green tea} & \wrong{darned} & awful \\
			& \wrong{south korea} & \wrong{asia} & advisors & \wrong{era} & cheeseburger & sundae & great & crappy \\
			& \wrong{markets} & alberta & regents & \wrong{tale} & patty & whipped cream & \wrong{suffered} & \wrong{honest} \\
			\midrule
			
			\multirow{5}{*}{\makecell{Anchored\\CorEx}} & \wrong{moscow} & \wrong{sports} & \wrong{republican} & \wrong{military} & \wrong{order} & \wrong{make} & \wrong{selection} & \wrong{did}   \\
			& british & \wrong{games} &  \wrong{senator} & \wrong{war} & \wrong{know} & chocolate & \wrong{prices} & \wrong{just}  \\
			& london & \wrong{players} & \wrong{democratic} & \wrong{troops} & \wrong{called} & \wrong{people} & great & \wrong{came}   \\
			& \wrong{german} & canadian & school & \wrong{baghdad} & fries & \wrong{right} & reasonable & \wrong{asked} \\
			& \wrong{russian} & coach & schools & \wrong{iraq} & \wrong{going} & \wrong{want} & \wrong{mac} & \wrong{table}  \\
			\midrule
			
			\multirow{5}{*}{\makecell{Labeled\\ETM}} & \wrong{france} & canadian & higher education & political & hamburger & pana & decent & horrible \\
			& \wrong{germany} & british columbia & educational & \wrong{expediency} & cheeseburger & gelato & great & terrible \\
			& \wrong{canada} & \wrong{britain} & school & \wrong{perceptions} & burgers & tiramisu & tasty & \wrong{good} \\
			& british & quebec & schools & foreign affairs & patty & cheesecake & \wrong{bad} & awful \\
			& \wrong{europe} & \wrong{north america} & regents & ideology & \wrong{steak} & ice cream & delicious & appallingly \\
			\midrule
			
			\multirow{5}{*}{\makecell{\CatEm}} &england & ontario & educational & political & burgers & dessert & delicious & sickening \\
			&london & toronto&  schools & international politics & cheeseburger & pastries & mindful & nasty\\
			&britons & quebec & higher education & liberalism & hamburger & cheesecakes & excellent & dreadful\\
			&scottish & montreal &secondary education  & political philosophy & burger king & scones & wonderful & freaks\\
			&great britain & ottawa & teachers &  geopolitics & smash burger & ice cream & faithful & cheapskates\\

			\bottomrule
		\end{tabular}
	}
	\vspace*{-1em}
\end{table*}

\setlength{\tabcolsep}{5pt}

\setlength{\tabcolsep}{2pt}
\begin{table}[ht]
	\centering
	\small
	\caption{Quantitative evaluation on discriminative topic mining.}
	\vspace*{-1em}
	\label{tab:quantity}
	\scalebox{0.92}{
		\begin{tabular}{c|cc|cc|cc|cc}
			\toprule
			\multirow{2}{*}{Methods} &
			\multicolumn{2}{c|}{\textbf{NYT}-Location} &
			\multicolumn{2}{c|}{\textbf{NYT}-Topic} & \multicolumn{2}{c|}{\textbf{Yelp}-Food} &
			\multicolumn{2}{c}{\textbf{Yelp}-Sentiment}\\
			& TC & MACC & TC & MACC & TC & MACC & TC & MACC \\
			\midrule
			LDA & 0.007 & 0.489 & 0.027 & 0.744 & -0.033 & 0.213 & -0.197 & 0.350 \\
			Seeded LDA & 0.024 & 0.168 & 0.031 & 0.456 & 0.016 & 0.188 & 0.049 & 0.223 \\
			TWE & 0.002 & 0.171 & -0.011 & 0.289 & 0.004 & 0.688 & -0.077 & 0.748 \\
			Anchored CorEx & 0.029 & 0.190 & 0.035 & 0.533 & 0.025 & 0.313 & 0.067 & 0.250 \\
			Labeled ETM & 0.032 & 0.493 & 0.025 & 0.889 & 0.012 & 0.775 & 0.026 & 0.852 \\
			\CatEm & \textbf{0.049} & \textbf{0.972} & \textbf{0.048} & \textbf{0.967} & \textbf{0.034} & \textbf{0.913} & \textbf{0.086} & \textbf{1.000} \\
			\bottomrule
		\end{tabular}
	}
	\vspace*{-2em}
\end{table}

\setlength{\tabcolsep}{5pt}

\noindent 
\textbf{Discussions.}
From Tables~\ref{tab:quality} and \ref{tab:quantity}, we observe that the standard topic model (LDA) retrieves reasonably good topics (even better than Seeded LDA and Anchored CorEx in some cases) relevant to category names, as long as careful manual selection of topics is performed. However, inspecting all the topics to select one's interested topics is inefficient and costly for users, especially when the number of topics is large. 

When users can provide a set of category names, seed guided topic modeling methods can directly retrieve relevant topics of user's interest, alleviating the burden of manual selection. Among the four guided topic modeling baselines, Seeded LDA and Anchored CorEx suffer from noisy retrieval results---some categories are dominated by off-topic terms. For example, Anchored CorEx retrieves words related to sports under the location category ``canada'', which should have been put under the topic category ``sports''. 

The above issue is rarely observed in the results of the other two embedding-based topic modeling baselines, TWE and Labeled ETM, because they employ distributed word representations when modeling topic word correlation, requiring the retrieved words to be semantically relevant to the category names. However, their results contain terms that are relevant but do not actually belong to the corresponding category. For example, Labeled ETM retrieves ``france'', ``germany'' and ``europe'' under the location category ``britain''. In short, TWE and Labeled ETM lacks \emph{discriminative} power over the set of provided categories, and fails to compare the relative \emph{generality/specificity} between a pair of terms (\eg, it is correct to put ``british'' under ``europe'', but incorrect vice versa).

Our proposed method \CatEm enjoys the benefits brought by word embeddings, and explicitly regularizes the embedding space to become discriminative for the provided set of categories. In addition, \CatEm learns the semantic specificity of each term in the corpus and enforces the words/phrases retrieved to be more specific than the category names. As shown in Tables~\ref{tab:quality} and \ref{tab:quantity}, \CatEm correctly retrieves distinctive terms that indeed belong to the category.

\subsection{Weakly-Supervised Text Classification}\label{sec:class}

In this subsection, we show that the discriminative power of \CatEm benefits document-level classification, and we explore the application of \CatEm to document classification under weak supervision.

Weakly-supervised text classification~\cite{Chang2008ImportanceOS,Meng2018WeaklySupervisedNT,Meng2019WeaklySupervisedHT,Song2014OnDH,zhang2019higitclass} uses category names or a set of keywords from each category instead of human annotated documents to train a classifier. It is especially preferable when manually labeling massive training documents is costly or difficult. 

Previous weakly-supervised document classification studies use unsupervised word representations to either retrieve from knowledge base relevant articles to category names as training data~\cite{Song2014OnDH}, or derive similar words and form pseudo training data for pre-training classifiers~\cite{Meng2018WeaklySupervisedNT,Meng2019WeaklySupervisedHT}. In this work, we do not propose new models for weakly-supervised document classification, but simply replace the unsupervised embeddings used in previous systems with \CatEm, based on the intuition that when the supervision is given on word-level, deriving discriminative word embeddings at the early stage is beneficial for all subsequent steps in weakly-supervised classification. 

In particular, we use WeSTClass~\cite{Meng2018WeaklySupervisedNT,Meng2019WeaklySupervisedHT} as the weakly-supervised classification model. WeSTClass first models topic distribution in the word embedding space to retrieve relevant words to the given category names, and applies self-training to bootstrap the model on unlabeled documents. It uses Word2Vec~\cite{Mikolov2013DistributedRO} as the word representation. In the following, we experiment with different word embedding models as input features to WeSTClass.

\begin{table*}[ht]
	\centering
	\caption{Weakly-supervised text classification evaluation based on WeSTClass~\cite{Meng2018WeaklySupervisedNT} model.}
    \vspace*{-1em}
	\label{tab:weak_class}
	\scalebox{1.0}{
		\begin{tabular}{c|cc|cc|cc|cc}
			\toprule
			\multirow{2}{*}{Embedding} &
			\multicolumn{2}{c|}{\textbf{NYT}-Location} &
			\multicolumn{2}{c|}{\textbf{NYT}-Topic} & \multicolumn{2}{c|}{\textbf{Yelp}-Food} &
			\multicolumn{2}{c}{\textbf{Yelp}-Sentiment}
			\\
			& Micro-F1 & Macro-F1 & Micro-F1 & Macro-F1 & Micro-F1 & Macro-F1 & Micro-F1 & Macro-F1 \\
			\midrule
			Word2Vec & 0.533 & 0.467 & 0.588 & 0.695 & 0.540 & 0.528 & 0.723 & 0.715 \\
			GloVe & 0.521 & 0.455 & 0.563 & 0.688 & 0.515 & 0.503 & 0.720 & 0.711 \\
			fastText & 0.543 & 0.485 & 0.575 & 0.693 & 0.544 & 0.529 & 0.738 & 0.743 \\
			BERT & 0.301 & 0.288 & 0.328 & 0.451 & 0.330 & 0.404 & 0.695 & 0.674 \\
			\CatEm & \textbf{0.655} & \textbf{0.613} & \textbf{0.611} & \textbf{0.739} & \textbf{0.656} & \textbf{0.648} & \textbf{0.838} & \textbf{0.836} \\
			\bottomrule
		\end{tabular}
	}
	\vspace*{-1em}
\end{table*}

\noindent
\textbf{Compared Methods.} We note here that our goal is \emph{not} designing a weakly-supervised classification method; instead, our purpose is to show that \CatEm benefits classification tasks with stronger discriminative power than unsupervised text embedding models by only leveraging category names. In this sense, our contribution is improving the input text feature quality for document classification when category names are available. To the best of our knowledge, this is the first work that proposes to learn discriminative text embedding only from category names (\ie, without requiring additional information other than the supervision given for weakly-supervised classification). We compare \CatEm with the following unsupervised text embedding baselines as input features to the state-of-the-art weakly-supervised classification model WeSTClass~\cite{Meng2018WeaklySupervisedNT,Meng2019WeaklySupervisedHT}. 
\begin{itemize}[leftmargin=*] 
\item Word2Vec~\cite{Mikolov2013DistributedRO}: Word2Vec
is a predictive word embedding model that learns distributed representations by maximizing the probability of using the center word to predict its local context words or in the opposite way.

\item GloVe~\cite{Pennington2014GloveGV}: GloVe
learns word embedding by factorizing a global word-word co-occurrence matrix where the co-occurrence is defined upon a fix-sized context window.

\item fastText~\cite{Bojanowski2017EnrichingWV}: fastText
is an extension of Word2Vec which learns word embedding efficiently by incorporating subword information. It uses the sum of vector representations of all n-grams in a word to predict context words in a fix-sized window.

\item BERT~\cite{Devlin2018BERTPO}:
BERT
is a state-of-the-art pretrained language model that provides contextualized word representations. It trains bi-directional Transformers~\cite{Vaswani2017AttentionIA} to predict randomly masked words and consecutive sentence relationships.

\end{itemize}

\noindent
\textbf{Evaluation Metrics.} We employ the Micro-F1 and Macro-F1 scores that are commonly used in multi-class classification evaluations~\cite{Meng2018WeaklySupervisedNT,Meng2019WeaklySupervisedHT} as the metrics.


\noindent
\textbf{Results.}
We first train all the embedding models on the corresponding corpus (except BERT which we take its pre-trained model and fine-tune it on the corpus), and use the trained embedding as the word representation to WeSTClass~\cite{Meng2018WeaklySupervisedNT,Meng2019WeaklySupervisedHT}. The weakly-supervised classification results on \textbf{NYT}-Location, \textbf{NYT}-Topic, \textbf{Yelp}-Food Type and \textbf{Yelp}-Sentiment are shown in Table~\ref{tab:weak_class}.

\smallskip
\noindent
\textbf{Discussions.} From Table~\ref{tab:weak_class}, we observe that: (1) Unsupervised embeddings (Word2Vec, GloVe and fastText) do not really have notable differences as word representations to WeSTClass;
(2) Despite its great effectiveness as a pre-trained deep language model for supervised tasks, BERT is not suitable for classification without sufficient training data, probably because BERT embedding has higher dimensionality (even the base model of BERT is $768$-dimensional) which might require stronger supervision signals to tune;
(3) \CatEm outperforms all unsupervised embeddings on \textbf{NYT}-Location and \textbf{Yelp}-Food Type and \textbf{Yelp}-Sentiment categories by a large margin, and have marginal advantage on \textbf{NYT}-Topic. This is probably because different locations (\eg, ``Canada'' vs. ``The United States''), food types (\eg, ``burgers'' vs. ``pizza''), and sentiment polarities (\eg, ``good'' vs. ``bad'') can have highly similar local contexts, and are more difficult to be differentiated than themes.
\CatEm explicitly regularizes the embedding space for the specific categories and becomes especially advantageous when the given category names are semantically similar. 

There have been very few previous efforts in the text classification literature that dedicate to learning discriminative word embeddings from word-level supervisions, and word embeddings are typically fine-tuned jointly with classification models~\cite{Kim2014ConvolutionalNN,Yang2016HierarchicalAN,Wang2018JointEO} via document-level supervisions. However, our study shows that under label scarcity scenarios, using word-level supervision only can bring significant improvements to weakly-supervised models. Therefore, it might be beneficial for future weakly-supervised/semi-supervised studies to also consider leveraging word-level supervision to gain a performance boost.

\subsection{Unsupervised Lexical Entailment Direction Identification}\label{sec:entail}
In \CatEm, we enforce the retrieved terms to be more specific than the given category name by comparing their learned distributional specificity values $\kappa$. Since $\kappa$ characterizes the semantic generality of a term, it can be directly applied to identify the direction in lexical entailment.

Lexical entailment (LE)~\cite{Vulic2017HyperLexAL} refers to the ``type-of'' relation, also known as hyponymy-hypernymy relation in NLP. LE typically includes two tasks: (1) Discriminate hypernymy from other relations (detection) and (2) Identify from a hyponymy-hypernymy pair which one is hyponymy (direction identification). Recently, there has been a line of supervised (\ie, require labeled hyponymy-hypernymy pairs as training data) embedding studies ~\cite{Nickel2017PoincarEF,Nickel2018LearningCH,Tifrea2019PoincareGH} that learn hyperbolic word embeddings to capture the lexical entailment relationships. 

In our evaluation, we focus on unsupervised methods for LE direction identification, which is closer to the application of \CatEm.

\noindent
\textbf{Compared Methods.} We compare \CatEm with the following unsupervised baselines that can identify the direction in a given hyponymy-hypernymy pair.
\begin{itemize}[leftmargin=*] 
\item Frequency~\cite{Weeds2004CharacterisingMO}: This baseline simply uses the frequency of a term in the corpus to characterize its generality. It hypothesizes that hypernyms are more frequent than hyponyms in the corpus.
\item SLQS~\cite{Santus2014ChasingHI}: SLQS measures the generality of a term via the entropy of its statistically most prominent context.
\item Vec-Norm: It is shown in \cite{Nguyen2017HierarchicalEF} that the L-2 norm of word embedding indicates the generality of a term, \ie, a general term tends to have a lower embedding norm, because it co-occurs with many different terms and its vector is dragged from different directions.
\end{itemize}

\noindent
\textbf{Benchmark Test Set.} Following~\cite{Santus2014ChasingHI}, we use the BLESS~\cite{Baroni2011HowWB} dataset for LE direction identification. BLESS contains $1,337$ unique hyponym-hypernym pairs. The task is to predict the directionality of hypernymy within each pair.

\noindent
\textbf{Results.} We train all models on the latest Wikipedia dump\footnote{\url{https://dumps.wikimedia.org/enwiki/latest/enwiki-latest-pages-articles.xml.bz2}} containing $2.4$ billion tokens and report the accuracy for hypernymy direction identification in Table~\ref{tab:le}.

\begin{table}[ht]
	\centering
	\caption{Lexical entailment direction identification.}
	\vspace*{-1em}
	\label{tab:le}
	\scalebox{1.0}{
		\begin{tabular}{c|cccc}
			\toprule
			Methods & Frequency & SLQS & Vec-Norm & \CatEm \\
			\midrule
			Accuracy & 0.659 & 0.861 & 0.562 & \textbf{0.895} \\
			\bottomrule
		\end{tabular}
	}
\vspace*{-1em}
\end{table}

\setlength{\tabcolsep}{5pt}
\begin{table*}[!htbp]
\centering
\caption{Coarse-to-fine topic presentation on \textbf{NYT}-Topic.}
\label{tab:kappa_nyt}
\scalebox{0.9}{
\begin{tabular}{|c|c|c|c|}
\hline
Range of $\kappa$ & Science ($\kappa_c = 0.539$) & Technology ($\kappa_c = 0.566$) & Health ($\kappa_c = 0.527$)\\
\hline
$\kappa_c < \kappa < 1.25\kappa_c$  & \makecell{scientist, academic, research, laboratory} 
	& \makecell{machine, equipment, devices, engineering} 
	& \makecell{medical, hospitals, patients, treatment}\\
\hline
$1.25\kappa_c < \kappa < 1.5\kappa_c$ & \makecell{physics, sociology,\\biology, astronomy}
	& \makecell{information technology, computing,\\telecommunication, biotechnology}
	&\makecell{mental hygiene, infectious diseases,\\hospitalizations, immunizations}\\
\hline
$1.5\kappa_c< \kappa < 1.75\kappa_c$ & \makecell{microbiology, anthropology,\\physiology, cosmology}
	& \makecell{wireless technology, nanotechnology,\\semiconductor industry, microelectronics}
	&\makecell{dental care, chronic illnesses,\\cardiovascular disease, diabetes}\\
\hline
$\kappa > 1.75\kappa_c$ & \makecell{national science foundation,\\george washington university,\\ hong kong university,\\american academy} 
	& \makecell{integrated circuits,\\assemblers,\\circuit board,\\advanced micro devices}
	&\makecell{juvenile diabetes,\\high blood pressure,\\family violence,\\kidney failure}\\
\hline
\end{tabular}
}
\end{table*}

\setcounter{subfigure}{0}
\begin{figure*}[h]
	\centering
	\subfigure[WordSim][Epoch 1]{
		\label{fig:e1}
		\includegraphics[width = 0.32\textwidth]{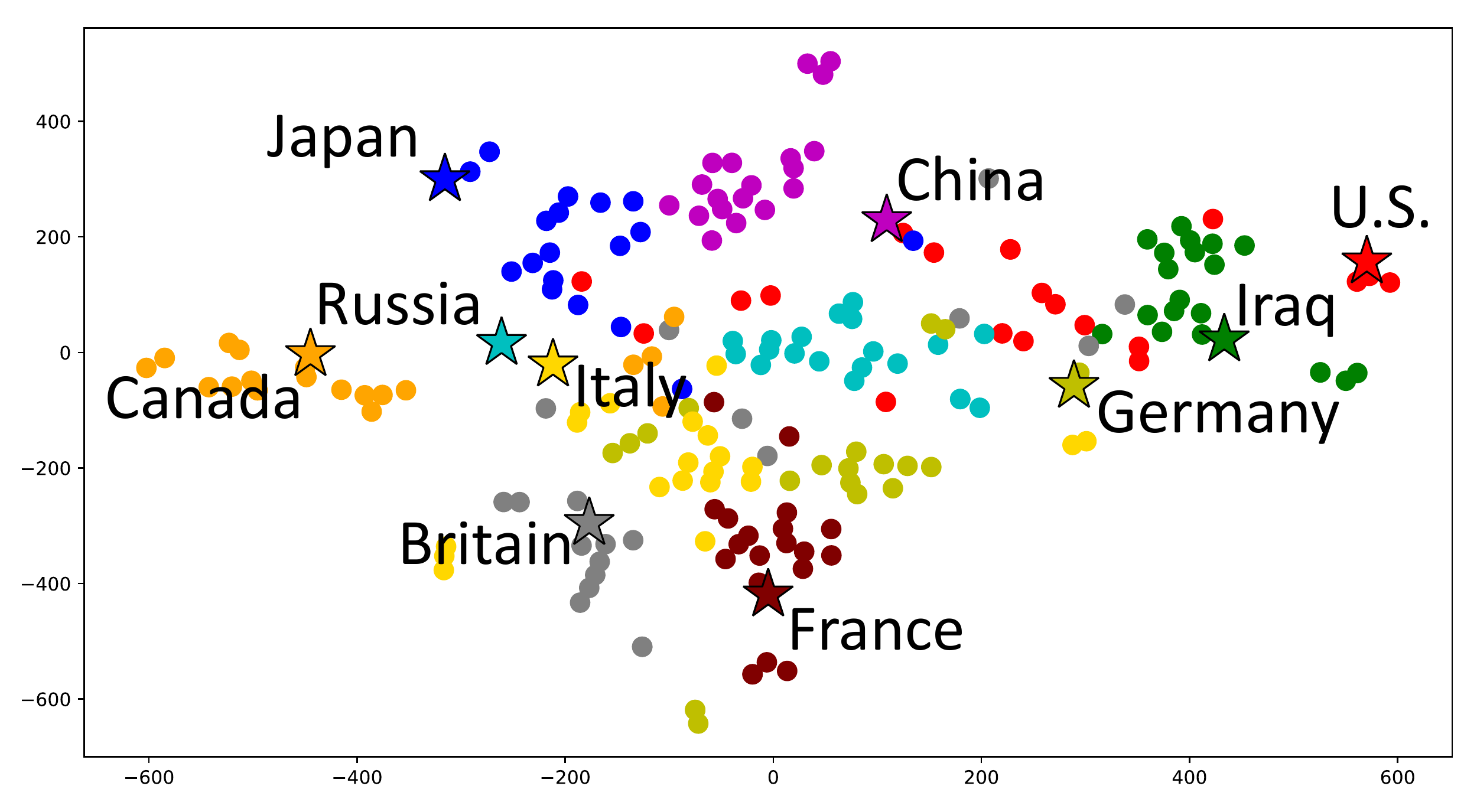}
	}
	\subfigure[WordSim][Epoch 3]{
		\label{fig:e3}
		\includegraphics[width = 0.32\textwidth]{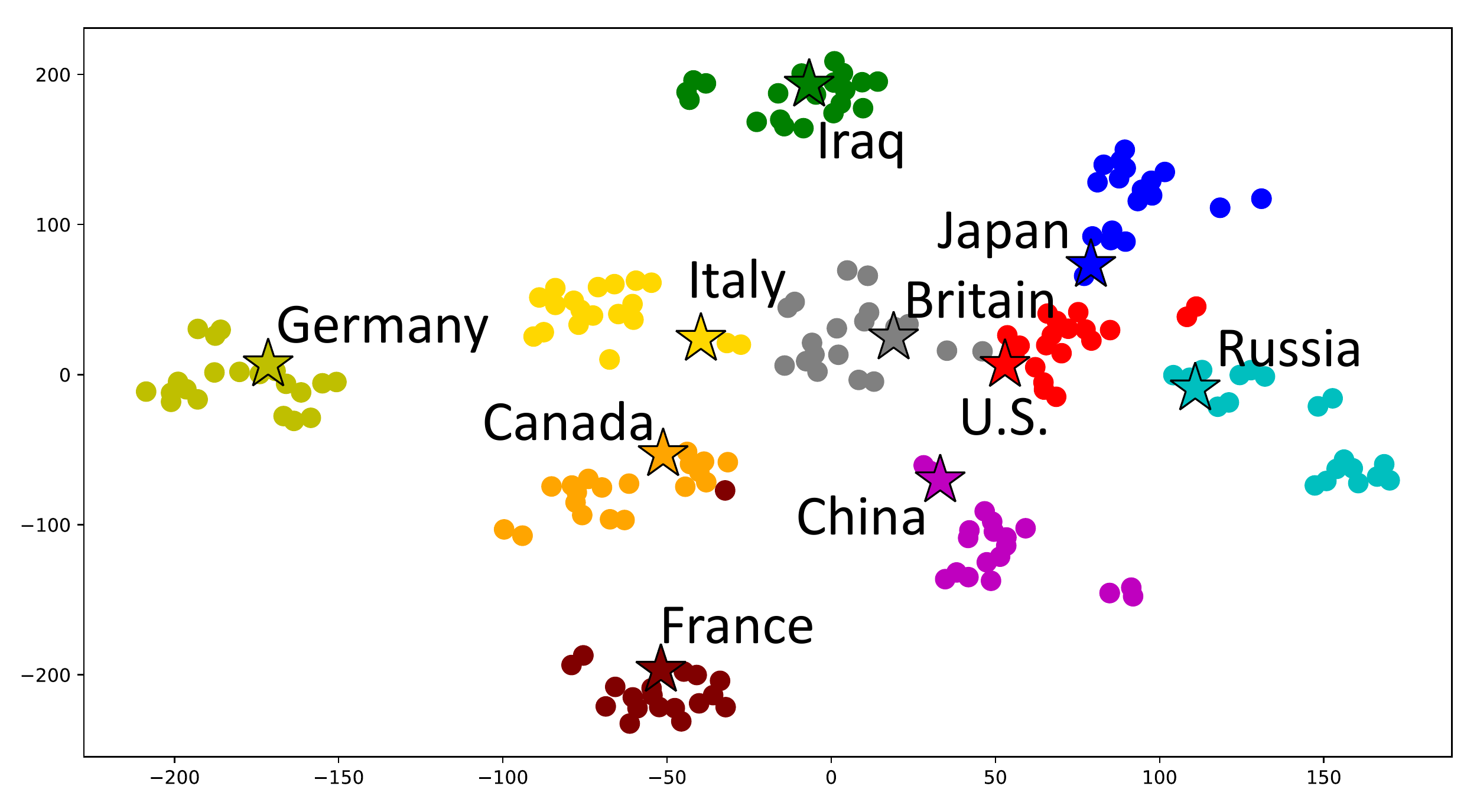}
	}
	\subfigure[WordSim][Epoch 5]{
		\label{fig:e5}
		\includegraphics[width = 0.32\textwidth]{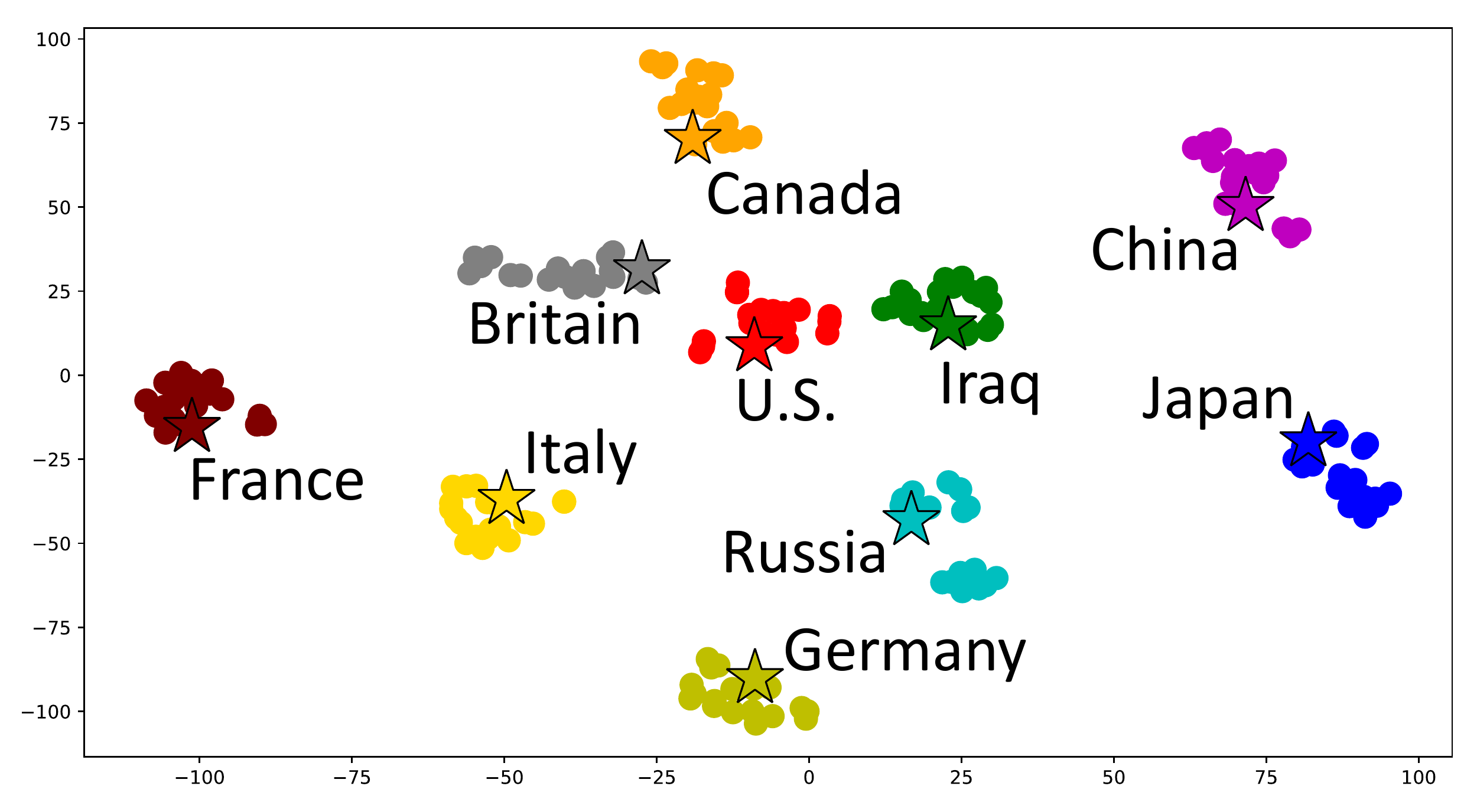}
	}
	\vspace*{-1em}
	\caption{Discriminative embedding space training for topic mining.}
	\label{fig:iter}
\end{figure*}

\noindent
\textbf{Discussions.} Our method achieves the highest accuracy on identifying the direction of lexical entailment among a pair of words, which explains the great effectiveness of \CatEm on retrieving terms that belong to a category. Another desirable property of \CatEm is that distributional specificity is jointly trained along with the text embedding, and can be directly obtained as a by-product. Our learning of word distributional specificity is based on the distributional inclusion hypothesis~\cite{ZhitomirskyGeffet2005TheDI} and has a probabilistic interpretation presented in Section~\ref{sec:explain}.

\subsection{Case Study}\label{sec:case}

\noindent
\textbf{Discriminative Embedding Space.}
In this case study, we demonstrate the effect of regularizing the embedding space with category representative words. Specifically, we apply t-SNE \cite{Maaten2008VisualizingDU} to visualize the embeddings trained on \textbf{NYT}-Location in Figure \ref{fig:iter} where category embeddings are denoted as stars, and the retrieved class representative phrases are denoted as points with the same color as their ground-truth corresponding categories. At the early stage of training (Figure \ref{fig:e1}), words from different categories are mixed together. During training, the categories are becoming well-separated. Category representative words gather around their corresponding category in the embedding space, which encourages other semantically similar words to move towards their belonging categories as well (Figure \ref{fig:iter} shows more words than class representative words retrieved by our model during training).

\noindent
\textbf{Coarse-to-Fine Topic Presentation.}
In the second set of case studies, we demonstrate the learned word distributional specificity with concrete examples from \textbf{NYT}-Topic, and illustrate how it helps present a topic in a coarse-to-fine manner. Table~\ref{tab:kappa_nyt} lists the most similar phrases with each category (measured by embedding cosine similarity) from different ranges of distributional specificity. When $\kappa$ is smaller, the retrieved words have wider semantic coverage.

A drawback of traditional topic modeling is that it presents each category via a top ranked list according to topic-word distribution, which usually seems randomly-ordered because latent probability distribution is generally hard to be interpreted by humans. In our model, however, one can sort the retrieved phrases under each topic according to distributional specificity, so that the topic mining results can be viewed in a coarse-to-fine manner.

%% file: 7-related.tex

\section{Related Work}

We review two lines of related work that are most relevant to our task: Topic modeling and task-oriented text embedding.

\subsection{Topic Modeling} 

Topic models discover semantically relevant terms that form coherent topics via probabilistic generative models. Unsupervised topic models have been studied for decades, among which pLSA \cite{Hofmann1999ProbabilisticLS} and LDA \cite{Blei2003LatentDA} are the most famous ones, serving as the backbone for many future variants. The basic idea is to represent documents via mixtures over latent topics, where each topic is characterized by a distribution over words. Subsequent studies lead to a large number of variants such as Hierarchical LDA \cite{griffiths2004hierarchical}, Correlated Topic Models \cite{blei2006correlated}, Pachinko Allocation \cite{li2006pachinko} and Concept Topic Models \cite{chemudugunta2008combining}. Although unsupervised topic models are sufficiently expressive to model multiple topics per document, they are unable to incorporate the category and label information into their learning procedure. 

Several modifications of topic models have been
proposed to incorporate supervision. Supervised LDA \cite{mcauliffe2008supervised} and DiscLDA \cite{lacoste2009disclda} assume each document is associated with a label and train the model by predicting the document category label. Author Topic Models \cite{rosen2004author} and Multi-Label Topic Models \cite{rubin2012statistical} further model each document as a bag of words with a bag of labels. However, these models obtain topics that do not correspond directly to the labels. Labeled LDA \cite{ramage2009labeled} and SSHLDA \cite{mao2012sshlda} can be used to solve this problem. However, all the \textit{supervised} models mentioned above requires sufficient annotated documents, which are expensive to obtain in some domains. In contrast, our model relies on very weak supervisions (\ie, a set of category names) which are much easier to obtain.

Several studies leverage word-level supervision to build topic models. For example, Dirichlet Forest~\cite{Andrzejewski2009LatentDA} has been used as priors to incorporate must-link and cannot-link constraints among seed words. Seeded LDA~\cite{Jagarlamudi2012IncorporatingLP} takes user-provided seed words as supervision to learn seed-related topics via a seed topic distribution. CorEx~\cite{Gallagher2017AnchoredCE} learns maximally informative topics from the corpus and uses total correlation as the measure. It can incorporate seed words by jointly compressing the text corpus and preserving seed relevant information. However, none of the above systems \textit{explicitly} model distinction among different topics, and they also do not require the retrieved terms to belong to the provided categories. As a result, the retrieved topics still suffer from irrelevant term intrusion, as we will demonstrate in the experiment section.

With the development of word embeddings \cite{Mikolov2013DistributedRO,Pennington2014GloveGV,Bojanowski2017EnrichingWV}, several studies propose to extend LDA to involve word embeddings. One common strategy is to convert the discrete text into continuous representations of embeddings, and then adapt LDA to generate real-valued data \cite{das2015gaussian,xun2016topic,batmanghelich2016nonparametric,Xun2017CollaborativelyIT}. There are a few other ways of combining LDA and embeddings. For example, \cite{nguyen2015improving} mixes the likelihood defined by LDA with a log-linear
model that uses pre-fitted word embeddings; \cite{xu2018distilled} adopts a geometric
perspective, using Wasserstein distances to learn topics and word embeddings jointly; \cite{Dieng2019TopicMI} uses the distributed representation of word embedding to enhance the robustness of topic models to rare words. Motivated by the success of these recent topic models, we model the text generation process in the embedding space, and propose several designs to tailor our model for the task of discriminative topic mining.

\subsection{Task-Oriented Text Embedding}

Discriminative text embeddings are typically trained under a supervised manner with task specific training data, such as training CNNs~\cite{Kim2014ConvolutionalNN} or RNNs~\cite{Yang2016HierarchicalAN} for text classification. Among supervised word embedding models, some previous studies are more relevant because they explicitly leverage the category information to optimize embedding for classification tasks. Predictive Text Embedding (PTE)~\cite{Tang2015PTEPT} constructs a heterogeneous text network and jointly embeds words, documents and labels based on word-word and word-document co-occurrences as well as labeled documents. 
Label-Embedding Attentive Model~\cite{Wang2018JointEO} jointly embeds words and labels so that attention mechanisms can be employed to discover category distinctive words. 
All the above frameworks require labeled training documents for fine-tuning word embeddings. Our method only requires category names to learn a discriminative embedding space over the categories, which are much easier to obtain.

Some recent studies propose to learn embeddings for lexical entailment, which is relevant to our task because it may help determine which terms belong to a category. Hyperbolic models such as Poincar\'e \cite{Nickel2017PoincarEF,Tifrea2019PoincareGH,Dhingra2018EmbeddingTI}, Lorentz~\cite{Nickel2018LearningCH} and hyperbolic cone~ \cite{Ganea2018HyperbolicEC} have proven successful in graded lexical entailment detection. However, the above models are supervised and require hypernym-hyponym training pairs, which may not be available under the setting of topic discovery. Our model jointly learns the word vector representation in the embedding space and its distributional specificity without requiring supervision, and simultaneously considers relatedness and specificity of words when retrieving category representative terms.


%% file: 8-concl.tex

\section{Conclusions and Future Work}

In this paper, we first propose a new task for topic discovery, discriminative topic mining, which aims to mine distinctive topics from text corpora guided by category names only. Then we introduce a category-name guided word embedding framework \CatEm that learns category distinctive text embedding by modeling the text generation process conditioned on the user provided categories. \CatEm effectively retrieves class representative terms based on both relatedness and specificity of words, by jointly learning word embedding and word distributional specificity. Experiments show that \CatEm retrieves high-quality distinctive topics, and benefits downstream tasks including weakly-supervised document classification and unsupervised lexical entailment direction identification.

In the future, we are interested in extending \CatEm to not only focus on user provided categories, but also have the potential to discover other latent topics in a text corpus, probably via distant supervision from knowledge bases. 
There are a wide range of downstream tasks that may benefit from \CatEm. For example,  we would like to exploit \CatEm for unsupervised taxonomy construction~\cite{Zhang2018TaxoGenCT} by applying \CatEm recursively at each level of the taxonomy to find potential children nodes. Furthermore, \CatEm might help entity set expansion via generating auxiliary sets consisting of relevant words to seed words~\cite{huang2020setcoexpan}.